\documentclass{article}

\usepackage{microtype}
\usepackage{graphicx}
\usepackage{adjustbox}
\usepackage{xfrac}
\usepackage{subfigure}
\usepackage{float} 
\usepackage{booktabs} 

\usepackage{hyperref}

\usepackage[accepted]{icml2023_template/icml2023}

\usepackage{amsmath}
\usepackage{amssymb}
\usepackage{mathtools}
\usepackage{amsthm}

\usepackage[capitalize,noabbrev]{cleveref}

\renewcommand{\algorithmicrequire}{\textbf{Input:}}
\renewcommand{\algorithmicensure}{\textbf{Output:}}

\usepackage{subfiles} 
\usepackage{amsthm}

\usepackage[thinc]{esdiff}

\newcommand{\E}{\mathrm{E}}
\newcommand{\I}{\mathrm{I}}

\newcommand{\R}{\mathbb{R}}

\newcommand{\be}{\begin{equation}}
\newcommand{\ee}{\end{equation}}
\newcommand{\bbmat}{\begin{bmatrix}}
\newcommand{\ebmat}{\end{bmatrix}}

\def\bea#1\eea{\begin{align}#1\end{align}}

\newtheorem*{theorem*}{Theorem}

\theoremstyle{plain}
\newtheorem{theorem}{Theorem}[section]

\theoremstyle{definition}
\newtheorem{definition}[theorem]{Definition}

\theoremstyle{remark}

\usepackage[textsize=tiny]{todonotes}

\begin{document}
\onecolumn 
\icmltitle{A Rank Stabilization Scaling Factor for Fine-Tuning with LoRA}



\icmlsetsymbol{equal}{}

\begin{icmlauthorlist}
\icmlauthor{Damjan Kalajdzievski}{equal,Tenyx}
\\\center Tenyx

\end{icmlauthorlist}

\icmlaffiliation{Tenyx}{Correspondence to damjan@tenyx.com.\\ \\Copyright 2023 Tenyx}




\vskip 0.3in



\printAffiliationsAndNotice{}  

\begin{abstract} As large language models (LLMs) have become increasingly compute and memory intensive, parameter-efficient fine-tuning~(PEFT) methods are now a common strategy to fine-tune LLMs. A popular PEFT method is Low-Rank Adapters (LoRA), which adds trainable low-rank ``adapters" to selected layers. Each adapter consists of a low-rank matrix product, multiplicatively scaled by a rank-dependent factor. This scaling factor, which divides adapters by a factor of the rank, results in slowed learning and stunted performance for LoRA with higher-rank adapters. Consequently, the use of LoRA in practice has generally been limited to very low ranks. In this work, we study the impact of the scaling factor on the learning process and prove that LoRA adapters should be divided by a factor of the square root of the rank. Modifying LoRA with the appropriate scaling factor, which we call the rank-stabilized LoRA (rsLoRA) method, easily provides for a fine-tuning compute/performance trade-off, where larger ranks can be used to trade off increased computational resources during training for better fine-tuning performance, with no change in inference computing cost.
\end{abstract}

\section{Introduction}
Large language models (LLMs) have become increasingly capable in the domain of natural language processing \cite{foundationmodels}. 
They have been successful in a wide variety of applications ranging from machine translation \cite{llmtranslationstudy}, disease prediction \cite{medllm}, generating code for robotics control policies \cite{llmrobotcode}, to chat-bot assistants \cite{rlhf}. 
While their inherent generalization capability is impressive, performance on down-stream tasks often requires fine-tuning \cite{deltatune}, which induces substantial computational resource requirements.

To address these problems, a multitude of fine-tuning approaches have recently been introduced for computationally-efficient training \cite{adapters,lora, bitfit,ia3,lokr,loha,deltatune}. These methods seek to optimize a reduced set of parameters while achieving comparable performance to full model fine-tuning. Of particular relevance for this paper is the method of Low-Rank Adapters (LoRA), in which ``adapters", consisting of a low-rank matrix product multiplied by a scaling factor, are added to a subset of parameter matrices of the pre-trained model to be optimized during fine-tuning.

In this paper we analyze the scaling factor of LoRA adapters. Our analysis proves that LoRA adapters should be divided by a factor of the square root of the rank, as opposed to conventional LoRA implementation in which adapters are divided by a factor of the rank. We call our method with this corrected scaling factor the rank-stabilized LoRA (rsLoRA) method. We experimentally verify the performance and learning stability of rsLoRA in comparison with the standard LoRA method. 
We illustrate that the conventional implementation causes gradient collapse as the rank increases, which slows the learning such that larger ranks perform no different than small ranks. In contrast, with rsLoRA the gradients do not collapse, and training with higher ranks increases performance. 
As such, the rsLoRA method easily provides
for a fine-tuning compute/performance trade off, where higher ranks can be used to trade increased training compute for
better performance. Further, since the adapters take the exact same form as in LoRA, 
there is no change in inference computational cost for different ranks.

\section{Background and Relevant Works}
We first provide an overview of the LoRA method and then follow with the introduction of a framework for studying scaling-initialization-update schemes used in \cite{scalinganalysis}, as our analysis follows a similar approach. 

\subsection{Low-Rank Adapters (LoRA)}
\label{section:lora}

In light of the hypothesis that fine-tuning of pre-trained LLM parameters takes place on a manifold with low intrinsic dimension \cite{peftintrinsicdim,intrinsicdim}, the authors of \cite{adapters} provide an alternative to full model fine-tuning, tuning far fewer parameters with comparable performance.
They introduce the concept of fine-tuning an LLM by fixing all existing pre-trained model parameters while adding an ``adapter" module after each pre-LayerNorm attention or feed-forward sub-module of the transformer. The adapters are composed of a two-layer neural network with a low hidden-layer dimension, such that learning is constrained to a parameter-efficient low-dimensional manifold.

\begin{figure}[h]
    \centering
    \begin{minipage}{1.0\linewidth} 
    \centering
    \adjustbox{trim={0\width} {.18\height} {0\width} {0\height},clip}
{\includegraphics[width=5.5cm]{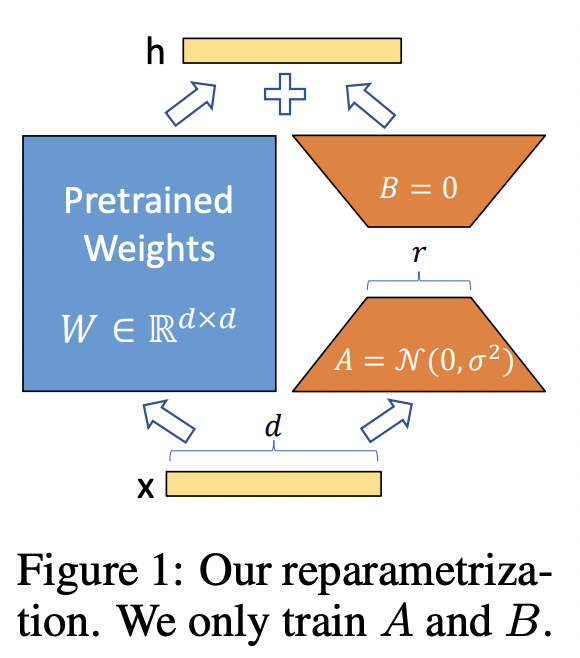}}
    \end{minipage}
    \caption{\textbf{LoRA adapters illustration} \cite{lora}.}
    \label{fig:lora}
\end{figure}

The LoRA method modifies the form of the adapters to be computed in parallel with their associated transformer sub-modules, such that following fine-tuning, they can be combined with the pre-trained parameters for efficient inference. Specifically, a linear sub-module of the pre-trained network with parameters $W\in\R^{d_2\times d_1},b\in\R^{d_2}$, which maps input $x_{\text{in}}\in\R^{d_1}$ as 
\begin{equation}
    x_{\text{out}} = Wx_{\text{in}}+b,
\end{equation}
is augmented 
by the addition of an adapter, consisting of parameters $A\in\R^{r\times d_1}$, $B\in\R^{d_2\times r}$, and a scaling factor $\gamma_r \in \R^{+}$. The resulting LoRA-augmented sub-module is defined by the mapping
\begin{equation}
    x_{\text{out}} = (W+\gamma_rBA)x_{\text{in}}+b.
\end{equation}
After fine-tuning, the single matrix $(W+\gamma_rBA)$ is stored and used in place of $W$, such that during inference there is no additional compute cost to the fine-tuned model.
Note that the adapter $\gamma_rBA$ of the sub-module is constrained to be of rank at most $r$, which is typically set to be much less than $d_1,d_2$. The matrices $B,A$ make up the free parameters to optimize during fine-tuning and are initialized such that $B=0_{d_2\times r}$, and entries of $A$ are iid with mean $0$ and variance 
which is independent of $r$.
The scaling factor $\gamma_r$ is some function of rank $r$ to account for the rank effects on the matrix product $BA$, and in the practiced LoRA method, $\gamma_r$ is set to $\gamma_r=\frac{\alpha}{r}$ for some hyperparameter $\alpha$. 

A follow-on method, AdaloRA \cite{adalora} allocates rank to LoRA adapters dynamically during training based on an available compute budget. It achieves this by parameterizing the adapter matrix product in an approximate SVD representation and iteratively prunes singular values (based on an importance score) to potentially reduce rank at each training step. Since AdaLoRA uses the same $\gamma_r$ as LoRA, while seeking to dynamically allow the use of different rank adapters, optimizing the selection of $\gamma_r$, as proposed in this paper, can improve upon AdaLoRA.

As will be shown, the setting of the scaling factor $\gamma_r$ in LoRA is overly aggressive and causes gradient collapse as the rank increases, which slows the learning such that LoRA fine-tuning using larger ranks performs no different than that with very small ranks. This may have led the authors of \cite{lora} to inaccurately conclude that very low ranks (e.g. 4,8,16) ``suffice" (\cite{lora} Table 6), since rank 64 showed no improvements in training with the same setting of $\alpha$. 

\subsection{Scaling-Initialization-Update Schemes}\label{section:learningscheme}
In order to derive the optimal scaling factor, we carried out a similar learning trajectory analysis to \cite{scalinganalysis}, where we consider the infinite width limit of the hidden dimension $r$.
In \cite{scalinganalysis} the authors study fully-connected multi-layer neural networks in the infinite width limit to analyze and draw conclusions about scaling-initialization-update schemes (which they refer to as abc-parametrizations). They define a scaling-initialization-update scheme for an $L$ layer network with hidden dimensions $d$, composed of weights $W^l$ for $l\leq L$, with the following parametrization of the initialization, learning rate, and parameters:
\begin{equation}\label{eqn:scalinganalysis}
\begin{split}
    &\text{The weights are scaled as }W^l=\frac{1}{d^{a_l}}\left(W^l_{i,j}\right)_{i,j},\\
    &\text{the initialization is }W^l_{i,j}\sim \mathcal{N}(0,1/d^{b_l})\text{ iid for each entry, and}\\
    &\text{the learning rate for updates is }\eta\frac{1}{d^{c}}\text{ for some scalar }\eta.
\end{split}
\end{equation}

They show that standard schemes, which only factor in $d$ to the initialization and set $a_l=c=0$, do not admit stable or non-collapsing learning for larger learning rates with larger $d$. They alleviate this with their scheme which sets $b_l = 1$ for all $l$, $c = 0$, $a_0 =-1/2,a_L = 1/2$, and $a_l = 0$ for all $0< l < L$. To arrive at these conclusions, they analytically solved for the learning dynamics in the linear network case. We take a similar approach to analyze the learning of the adapters of LoRA with respect to the scaling factor, from which we obtain in section \ref*{section:method} theorem \ref{mainthrm} and our rsLoRA method.

\section{rsLoRA: Rank-Stabilized Adapters}\label{section:method}

With LoRA adapters, scaling the matrix product $BA$ with $\gamma_r$ (which depends on rank $r$), affects the learning trajectory. Specifically, one needs to ensure the choice of $\gamma_r$ is ``correct" in the sense that the matrix $\gamma_rBA$ is stable for all ranks throughout the learning process, but that $\gamma_r$ is not overly aggressive so as to collapse or slow down learning (see definition \ref{def:rankstabilized}).
Indeed, if one chooses $\gamma_r=\frac{\alpha}{r}$  there is minimal impact on the fine-tuning loss when varying the rank $r$ (see figure \ref{fig:training}). 
This is not ideal, since learning with higher ranks could offer better performance when larger computational resources are available. Moreover, higher ranks only impose extra cost in training and not inference, as once training is completed the adapters are added to the pre-trained weights to obtain a single matrix $(W+\gamma_rBA)$ which is used in inference.

In order to precisely define what we mean for an adapter to be stable with respect to rank, we define a ``rank-stabilized" adapter as follows:
\begin{definition}\label{def:rankstabilized}
    An adapter $\gamma_rBA$ is \textbf{rank-stabilized} if the following two conditions hold:
    \begin{enumerate}
        \item If the inputs to the adapter are iid such that the $m$'th moment is $\Theta_r(1)$ in each entry, then the $m$'th moment of the outputs of the adapter is also $\Theta_r(1)$ in each entry.
        \item If the gradient of the loss with respect to the  the adapter outputs are $\Theta_r(1)$ in each entry, then the loss gradients into the input of the adapter are also $\Theta_r(1)$ in each entry.
    \end{enumerate}
\end{definition}

Using analysis of the limiting $r\to\infty$ case, we show that the only setting of $\gamma_r$ (modulo an additive constant) which results in rank-stabilized adapters is
\begin{equation}
    \gamma_r= \frac{\alpha}{\sqrt r}
\end{equation}
for some hyperparameter $\alpha$. We call our method, which uses the above setting of $\gamma_r$, the rank-stabilized LoRA (or rsLoRA) method.

\begin{theorem}\label{mainthrm}
    Consider LoRA adapters which are of the form $\gamma_r BA$, where $B\in \R^{d_1\times r}, A\in \R^{r\times d_2}$ are initialised such that $B$ is $0_{d_1\times r}$, entries of $A$ are iid with mean $0$ and variance $\sigma_A$ not depending on $r$, and $\gamma_r \in \R$ is such that $\gamma_r\underset{r\to \infty}{\longrightarrow}0$. 
    
    In expectation over initialization, all adapters are rank-stabilized if and only if \begin{equation}
        \gamma_r\in\Theta_r(\frac{1}{\sqrt{r}}).
    \end{equation}
    In particular, the above holds at any point in the learning trajectory, and unless $\gamma_r\in\Theta_r(\frac{1}{\sqrt{r}}),$ there is unstable or collapsing learning for sufficiently large values of $r$. 
\end{theorem}

See appendix \ref{appendix:proof} for detailed proof. The above theorem suggests that the only way to ensure stability through the adapters for an arbitrary rank, is to scale the adapters with $\gamma_r\propto\frac{1}{\sqrt{r}}.$ At first glance, it may appear that the assumptions made for the definition of rank-stabilized adapters are not straightforward, since for a layer in the middle of the model the gradient to the adapter should depend on $r$ via adapters from later layers. Moreover, the activations should depend on $r$ from adapters of earlier layers. However, since the input data to the model does not depend at all on the rank $r$ of adapters, we can inductively apply the theorem from input layers to output layers with the forward propagation of activations. This results in an output loss gradient above the adapters which should not otherwise have reason to have magnitude unstable or collapsing in the limit of $r$. Following this gradient backpropagation, we can inductively apply the theorem in a similar fashion from output layers to input layers to pass through each adapter while maintaining the gradient magnitude in $\Theta_r(1)$.

We would like to point out that theorem \ref{mainthrm} pertains to the stability or collapse of the learning process in the limiting case of $r$, and does not make any statements about possible variations in performance of features learned with different ranks when learning is stable. The quality of the features during learning can also affect the magnitude of the gradient. This suggests that if the quality of features learned consistently depends on $r$, the assumptions of the theorem may not hold. Conversely, if the quality of learning is independent of rank, there would be no reason to use larger ranks given the increased resources required. Therefore, the extent to which these factors interplay, the implications of the theorem in practice, and potential benefits of using the corrected scaling factor, require experimental validation and examination.

\section{Experimental Results}\label{section:experiments}
Here we present experimental results that show the benefits of rsLoRA when compared to LoRA, and verify the practical consequences of the theorem stating that unless $\gamma_r\in\Theta_r(\frac{1}{\sqrt{r}}),$ there is unstable or collapsing learning for sufficiently large ranks. 

\begin{figure}[h]
    \centering
    \begin{minipage}{.55\linewidth}
        \includegraphics[width=1.0\linewidth]{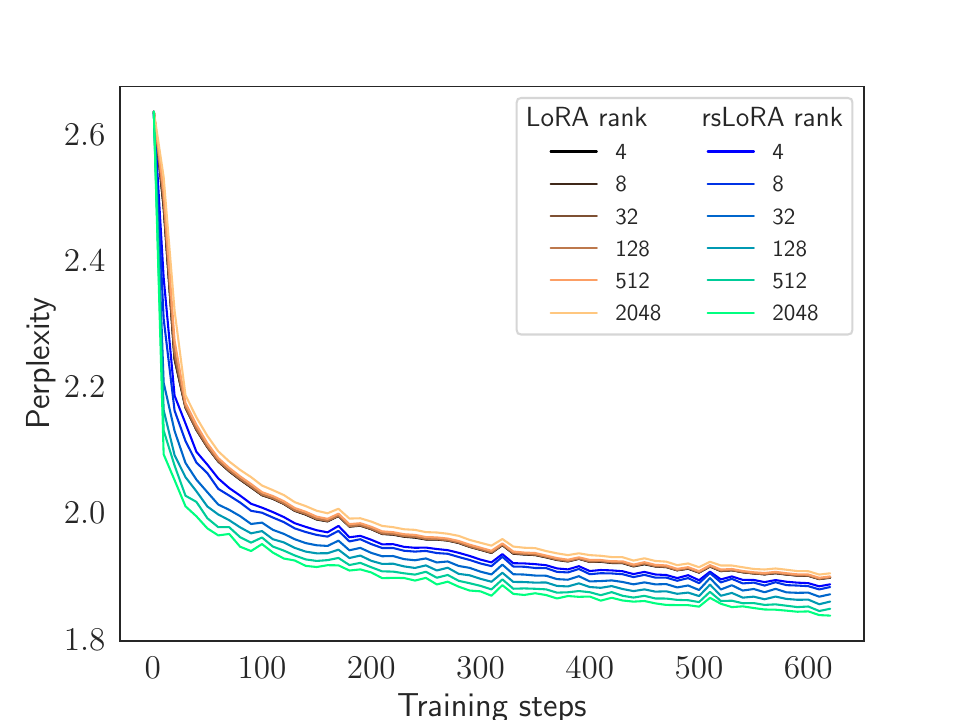}
    \end{minipage}
    \caption{\textbf{Fine-tuning perplexity for varying rank}.
    We vary ranks $r\in\{4,8,32,128,512,2048\}$ and compare the fine-tuning performance of LoRA and rsLoRA, where higher ranks are plotted using brighter colors. We observe that LoRA models, denoted in the copper color gradient, converge to a similar loss irrespective of the rank of the adapter, with some larger ranks even performing slightly worse. In contrast, we see that our proposed method rsLoRA, denoted in the blue-to-green color gradient, unlocks improved fine-tuning performance with higher ranks. 
        }
    \label{fig:training}
\end{figure}

To carry out our experiments with LoRA and rsLoRA, we choose a popular model and fine-tuning dataset: We fine-tune the Llama 2 model \cite{llama2} on 20,000 examples of the OpenOrca instruction tuning dataset \cite{orca}, using the AdamW optimizer \cite{adamw} with the HuggingFace default learning rate of .00005 on a constant learning rate schedule.
We add and optimize adapters in all linear (i.e., non-LayerNorm) attention and feed-forward MLP sub-modules of the transformer, since this has been shown to perform best with LoRA for a given parameter number budget (\cite{adalora} Appendix F). 
To observe the effects of varying adapter rank, we train with each ranks $r\in\{4,8,32,128,512,2048\}$. 
To directly measure the stability of learning throughout training, we track the average parameter gradient norm (Figure \ref{fig:gradient}).

The study \cite{deltatune} asserts that fine-tuning on an increased number of parameters tends to perform better, with full-model fine-tuning consistently outperforming parameter efficient methods. Therefore we have reason to conjecture that training with larger ranks should outperform training with smaller ranks. Indeed, as illustrated in figure \ref{fig:training}, we find that rsLoRA unlocks this performance increase for larger ranks, while LoRA's overly aggressive scaling factor collapses and slows learning with larger ranks such that there is little to no performance difference when compared to low ranks.

Validating our predictions, we illustrate in figure \ref{fig:gradient} that LoRA has collapsing gradients with higher ranks, whereas rsLoRA maintains the same gradient norm for each rank at the onset of training, while the norms remain approximately within the same order of magnitude throughout the training process.

\begin{figure}[h]
    \centering
    \begin{minipage}{.6\linewidth}
        \includegraphics[width=1.0\linewidth]{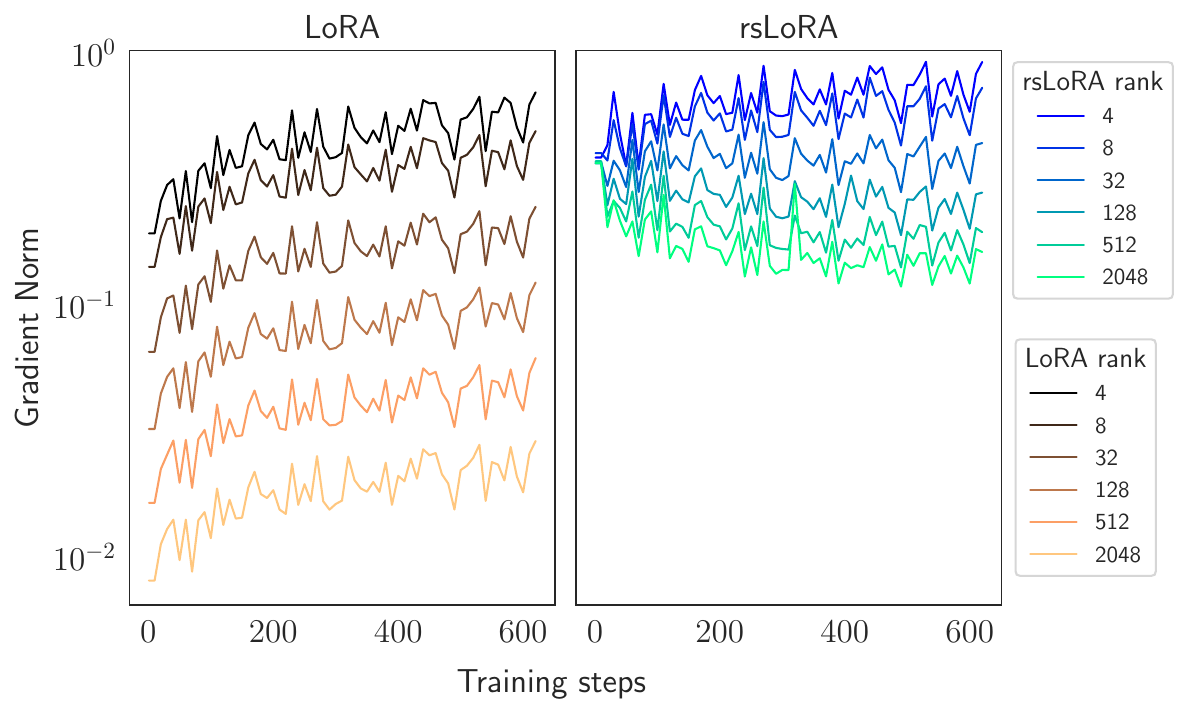}
    \end{minipage}
    \caption{\textbf{Gradient norm for varying rank}. 
    To evaluate learning stability, we vary ranks $r\in\{4,8,32,128,512,2048\}$ and compare the average parameter gradient norm of LoRA and rsLoRA, where higher ranks are line-plotted brighter in color. In the copper color gradient, LoRA has collapsing gradients with higher rank. In the blue to green color gradient, rsLoRA has the same gradient norm for each rank at the onset of training, while the norms remain approximately within the same order of magnitude throughout training.
        }
    \label{fig:gradient}
\end{figure}

To account for other potential variables, we run an extensive set of ablations, some of which we enumerate here (see appendix section \ref{appendix:ablations} for all ablations and details): 
\begin{itemize}
    \item We repeat the experiment with a different pre-trained model, dataset and optimizer, and observe the same pattern of results.
    \item We observe the same pattern of stability in gradient norms with SGD, to control for any additional stabilization effects that may have been caused by AdamW or other adaptive optimizers used in practice.
    \item We sweep through learning rates for LoRA with rank 4, and get that performance cannot match that of rsLoRA with high rank with the default learning rate, showing that the increased capacity is a requirement for improved learning, and the larger $\gamma_r$ of rsLoRA is not just acting as a learning rate boost.
    \item Using adapters in just the attention modules, which is commonly used with LoRA, gives similar results.
    \item If instead of reparameterizing the adapters with a scaling factor we only rank-correct the initialization of the adapter, we observe that learning becomes unstable during training for larger ranks.
    
\end{itemize}

\section{Conclusion}
In this paper we theoretically derived a rank correcting factor for scaling additive adapters when fine-tuning, and experimentally validated the approach. We showed that, in direct contrast to LoRA, learning with the proposed rank-stabilized adapters method remains stable and non-collapsing even for very large adapter ranks, allowing one to achieve better performance using larger ranks. This allows for 
unrestricted parameter efficiency, where one can use the maximum rank possible given an available memory budget to obtain the best fine-tuning performance. 
Our findings motivate further studies into the effects of the learning manifold intrinsic dimensionality in the context of fine tuning, since the implications of the original LoRA work inaccurately suggests and purveys the idea that very low ranks suffice when aiming to maximize fine-tuning performance. 

Possible avenues for future work include investigation of the proposed correction factor's use in the AdaLoRA method \cite{adalora}, which currently employs the same scaling factor for adapters as in conventional LoRA. Since any adapter is free to potentially be of full rank, the method may be more useful when built on rsLoRA instead of LoRA. Based on our results, we conjecture that AdaLoRA with rsLoRA should see increased fine-tuning performance, with substantial improvements in the case of higher rank budgets.

\bibliography{bibliography}
\bibliographystyle{icml2023_template/icml2023}

\newpage
\appendix
\onecolumn

\section{Proof of theorem \ref{mainthrm}}\label{appendix:proof}

\begin{theorem*}
Let the LoRA adapters be of the form $\gamma_r BA$, where $B\in \R^{d_1\times r},A\in \R^{r\times d_2}$ are initialised such that $B$ is initially $0_{d_1\times r}$, entries of $A$ are iid with mean $0$ and variance $\sigma_A$ not depending on $r$, and $\gamma_r \in \R$ is such that $\gamma_r\underset{r\to \infty}{\longrightarrow}0$. 
    
    In expectation over initialization, assuming the inputs to the adapter are iid distributed such that the $m$'th moment is $\Theta_r(1)$ in each entry, we have that the $m$'th moment of the outputs of the adapter is $\Theta_r(1)$ in each entry if and only if $$\gamma_r\in\Theta_r(\frac{1}{\sqrt{r}}).$$
    
    In expectation over initialization, assuming the loss gradient to the adapter outputs are $\Theta_r(1)$ in each entry, we have that the loss gradients into the input of the adapter are $\Theta_r(1)$ in each entry if and only if $$\gamma_r\in\Theta_r(\frac{1}{\sqrt{r}}).$$
    In particular, the above holds at any point in the learning trajectory if the assumptions do, and unless $\gamma_r\in\Theta_r(\frac{1}{\sqrt{r}}),$ there is unstable or collapsing learning for $r$ large enough.
\end{theorem*}

\begin{proof}[Proof of theorem \ref{mainthrm}]
\phantom\qedhere

Let $f(x)=\gamma_rBAx$, and $\mathcal{L}(f(x))$ denote the loss. let $B_n,A_n$ denote $B,A$ after the $n'th$ SGD update on input $x_n$ with learning rate $\eta$. Recall that $B_0=0_{d\times r}$, and see that for $v_n=\nabla_{f(x_n)}\mathcal{L}(f(x_n))$:
\begin{equation}
\begin{split}\label{grads}
    &\nabla_{B_n} \mathcal{L}=\gamma_rv_nx_n^TA_n^T,\\
    &\nabla_{A_n} \mathcal{L}=\gamma_rB_n^Tv_nx_n^T.
\end{split}
\end{equation}
First by induction for $n\geq 1$ check that 
\begin{equation}
\begin{split}\label{induction}
    &B_n=(-\eta\gamma_r\sum_{k=0}^{n-1}v_kx_k^T + \mathcal{O}_r(\gamma_r^2))A_0^T\\
    &A_n=A_0(1+\mathcal{O}_r(\gamma_r^2)).
\end{split}
\end{equation}

It follows that after $n$ updates we have 
\begin{equation}
\gamma_rB_nA_n=-\gamma_r^2\eta\sum_{k=0}^{n-1}v_kx_k^TA_0^TA_0 + \mathcal{O}_r(\gamma_r^3)A_0^TA_0
\end{equation}

If we take the expectation of the initialization $A_0$, noting that $$\E_{A_0}(A_0^TA_0)=r\sigma_A\I_{dxd},$$ we get

\begin{equation}
    \E_{A_0}(\gamma_rB_nA_n)= -\gamma_r^2r\sigma_A\eta\sum_{k=0}^{n-1}v_kx_k^T + \mathcal{O}_r(\gamma_r^3r).
\end{equation}

(Backward pass:) On a new input $x_n$, the gradient
\begin{equation}
    \nabla_{x_n}\mathcal{L}(\gamma_rB_nA_nx_n)=-\gamma_r^2r\sigma_A\eta\sum_{k=0}^{n-1}x_kv_k^Tv_n + \mathcal{O}_r(\gamma_r^3r)\in\Theta_r(\gamma_r^2r).
\end{equation}

(Forward pass:) On a new input $x_n$, use the assumption that the inputs are iid with moments $\Theta_r(1)$ to note 
$\E_x((x_k^Tx_n)^m)\in\Theta_r(1)$. Then, the assumptions $\gamma_r\to 0$ and the above numbered equation give
\begin{equation}
\begin{split}
    \E_{x,A_0}((\gamma_rB_nA_nx_n)^m)&= (-\gamma_r^2r\sigma_A\eta)^m\sum_{k=0}^{n-1}v_k^m\E_x((x_k^Tx_n)^m) + \mathcal{O}_r((\gamma_r^3r)^m)\in\Theta_r((\gamma_r^2r)^m).
\end{split}
\end{equation}
For this to not be unstable or collapse in the limit of $r\to \infty$, we need 
$\Theta_r((\gamma_r^2r)^m)=\Theta_r(1)$ or equivalently $$\gamma_r\in\Theta_r(\frac{1}{\sqrt{r}}).$$
\end{proof}

\section{Ablations and additional experiments}\label{appendix:ablations}
For all experiments including those in the main text, we use a batch size of 32 and a context window size of 512.
\subsection{Stability with SGD}
To rule out any stabilizing effects from AdamW, and since the theorem is done for SGD, we check the parameter gradient norm stability for rsLoRA vs LoRA when training with SGD. We pick a high but stable learning rate of .0001 and observe a similar pattern of stability as in figure \ref{fig:gradient}:
\begin{figure}[h]
    \centering
    \begin{minipage}{.55\linewidth}
        \includegraphics[width=1.0\linewidth]{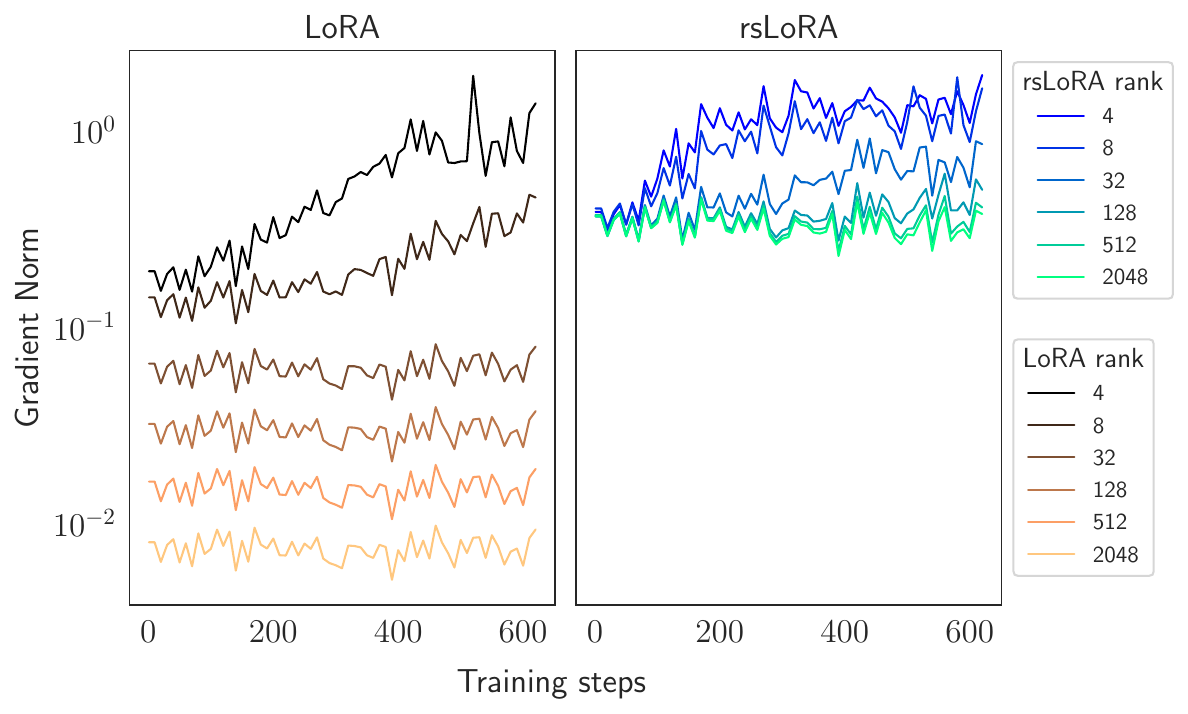}
    \end{minipage}
    \caption{\textbf{Gradient norm for varying rank with SGD}. 
    To observe learning stability, we vary ranks $r\in\{4,8,32,128,512,2048\}$ and compare the average parameter gradient norm of LoRA and rsLoRA, where higher ranks are line-plotted brighter in color. We also shade in the region in between the minimum and maximum rank for each method in its corresponding color scheme. In the copper color gradient, LoRA gradient has collapsing gradients with higher rank. In the blue to green color gradient, rsLoRA has the same gradient norm for each rank at the start of training, and the norms remain within the same order of magnitude throughout training. We observe even greater stability of rsLoRA with SGD, as all ranks stay at nearly the same norm value for about 100 steps into the learning trajectory.
        }
    \label{appendix:fig:sgdgradient}
\end{figure}

\subsection{Change of Model/Optimizer/Dataset}

To show the generality of the results, we re-run training with a different model, optimizer, and fine-tuning dataset. We train the 6 billion parameter GPT-J \cite{gptj} on the GSM8k benchmark dataset, which consists of grade school math word problems \cite{gsm8k}. We use the adaptive optimizer Adafactor \cite{adafactor}. The results plotted below show that our findings do indeed translate to this completely different setting.

\begin{figure}[!h]
    \centering
    \begin{minipage}{.45\linewidth}
        \includegraphics[width=1.0\linewidth]{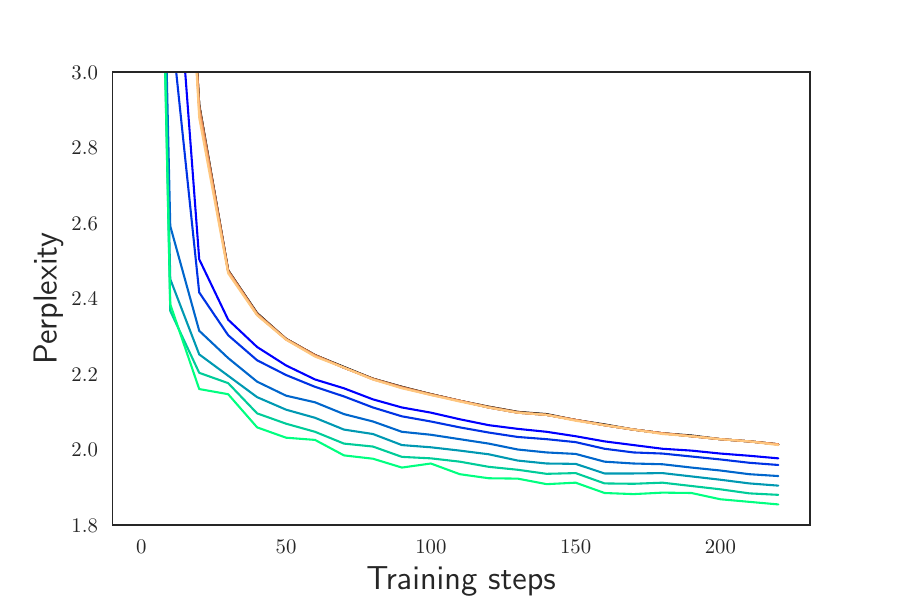}
    \end{minipage}
    \begin{minipage}{.45\linewidth}
        \includegraphics[width=1.0\linewidth]{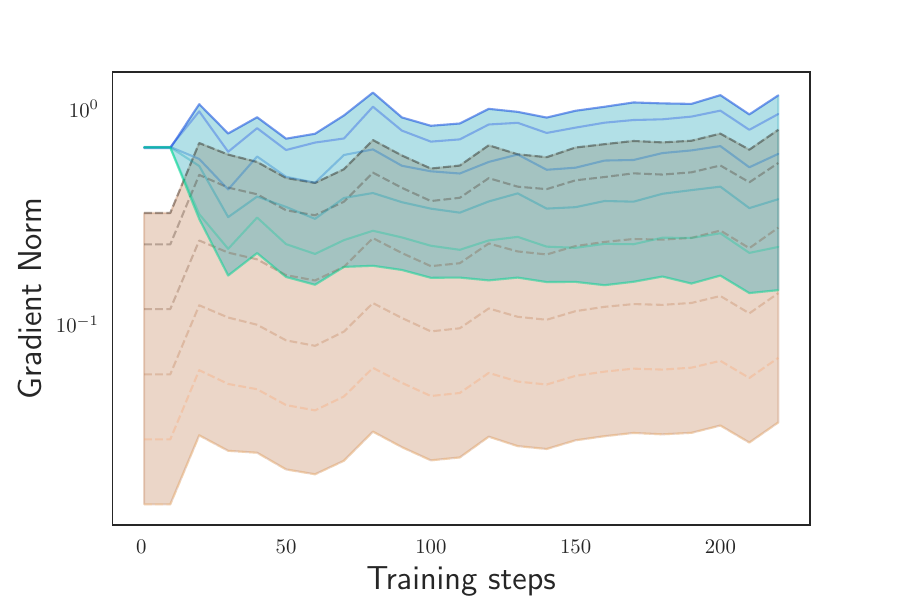}
    \end{minipage}
    \caption{\textbf{GPT-J finetuned on GSM8k with Adafactor}. We vary ranks $r\in\{4,8,32,128,512,2048\}$. Blue-green gradient is used for rsLoRA, and copper for LoRA, where for each, brighter color indicates higher rank. In the gradient norm plot, we also shade in the region in between the minimum and maximum rank for each method in its corresponding color scheme. We see the same pattern of results, where training for larger ranks does not change performance with LoRA (in fact, here all LoRA training curves overlap to look like a single run) but unlocks additional fine-tuning performance with rsLoRA. We also see that the parameter gradient norm for each rank is similar with rsLoRA and different orders of magnitude with LoRA, which collapses learning at higher ranks.
    }
    \label{fig:gptjtraining}
\end{figure}

\newpage

\subsection{Scaling at Initialization Only}
Here we examine training performance where instead of reparameterizing the adapters with a scaling factor, we only rank-correct the initialization of the adapter. We examine this in two ways by looking at rank 4 and 2048: First we do not use any scaling factor at all and just scale the initialization of the matrix $A$ by $\frac{1}{\sqrt{r}}$.
With this, although the gradient norms are the same magnitude at initialization for rank 4 and 2048, we observed that training perplexity becomes unstable during training for rank 2048.

Secondly, we scale the initialization of the matrix $A$ by $\frac{1}{\sqrt{r}}$ but also use the LoRA scaling factor of $\frac{1}{r}$ to scale adapters. This shows the same sub-optimal training as with LoRA, and that the initialization cannot correct for this.

\begin{figure}[!h]
    \centering
    \begin{minipage}{.45\linewidth}
        \includegraphics[width=1.0\linewidth]{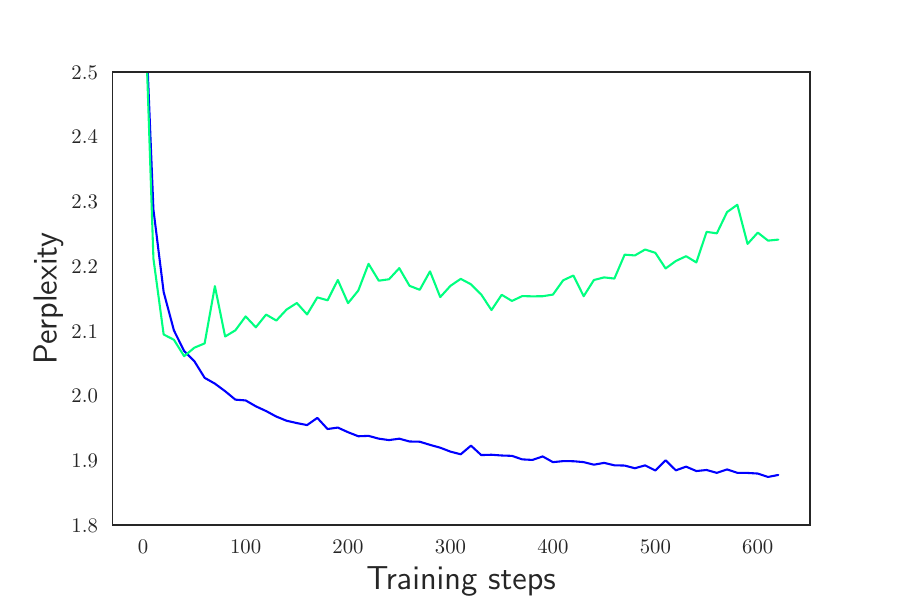}
    \end{minipage}
    \begin{minipage}{.45\linewidth}
        \includegraphics[width=1.0\linewidth]{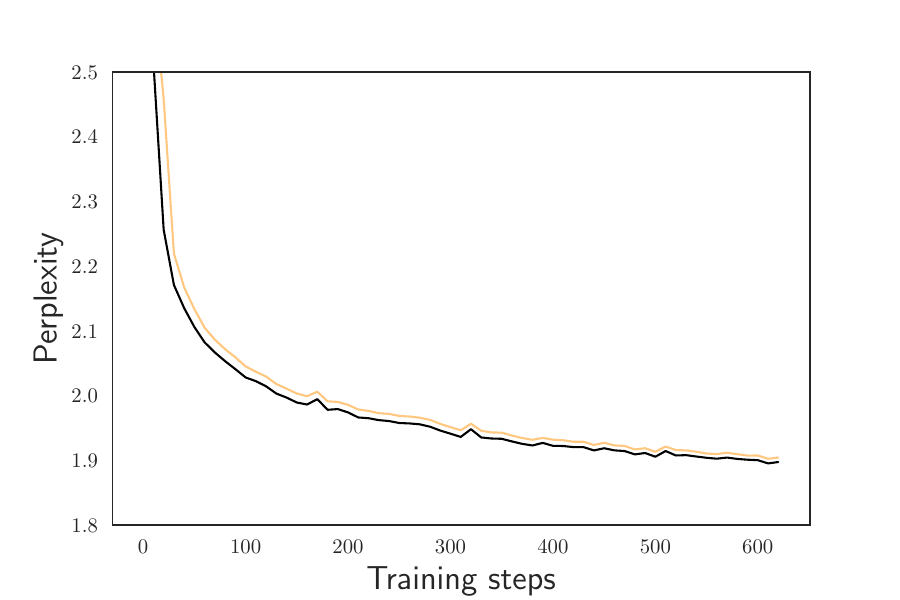}
    \end{minipage}
    \caption{\textbf{Training with scaled initialization} for ranks 4 (darker) and 2048 (brighter). On the left we plot the result of not using any scaling factor at all and just scaling the initialization by $\frac{1}{\sqrt{r}}$. We observe that learning becomes unstable during training for rank 2048. On the right we plot the result of scaling the initialization but use the LoRA adapter scaling factor of $\frac{1}{r}$. This shows the same sub-optimal learning for high rank as training with LoRA, and that the initialization cannot correct for this.}
    \label{fig:init}
\end{figure}

\newpage

\subsection{Learning Rate for Low Rank is Not Enough}
To definitively show that higher ranks are required for better learning, and that the larger $\gamma_r$ of rsLoRA is not just acting as a learning rate boost, we sweep through learning rates in $\{5\times 10^{n}:-5<n<-1\}\cup\{1\times 10^{n}:-5<n<-1\}$ for LoRA with rank 4, and compare to rank 2048 with rsLoRA. We get that the best performance with rank 4 adapters (at learning rate $.0005$), reaching 1.863 perplexity, cannot match that of rsLoRA with high rank with the default learning rate, which reaches 1.836 perplexity, showing that rsLoRA and the increased capacity of higher ranks is a requirement for improved learning.

\subsection{Adapters in Attention Modules Only}

We ablated using adapters in all modules and observed that putting rsLoRA in the attention modules only results in a similar effect; The larger rank still trains better for rsLoRA and does not for LoRA, and the gradient norms of trained parameters are still collapsing in the same way with LoRA when rank is large. 

\begin{figure}[!h]
    \centering
    \begin{minipage}{.45\linewidth}
        \includegraphics[width=1.0\linewidth]{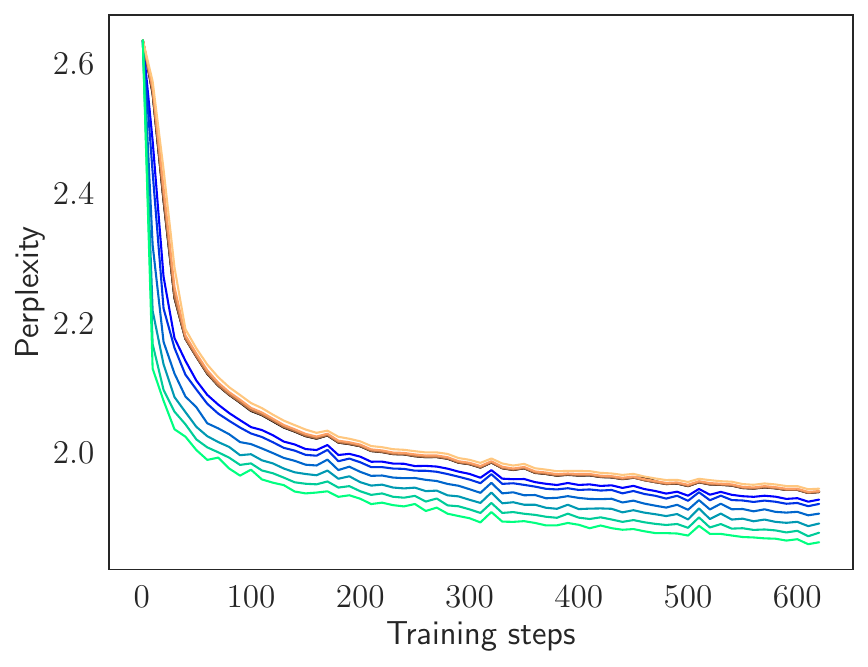}
    \end{minipage}
    \begin{minipage}{.45\linewidth}
        \includegraphics[width=1.0\linewidth]{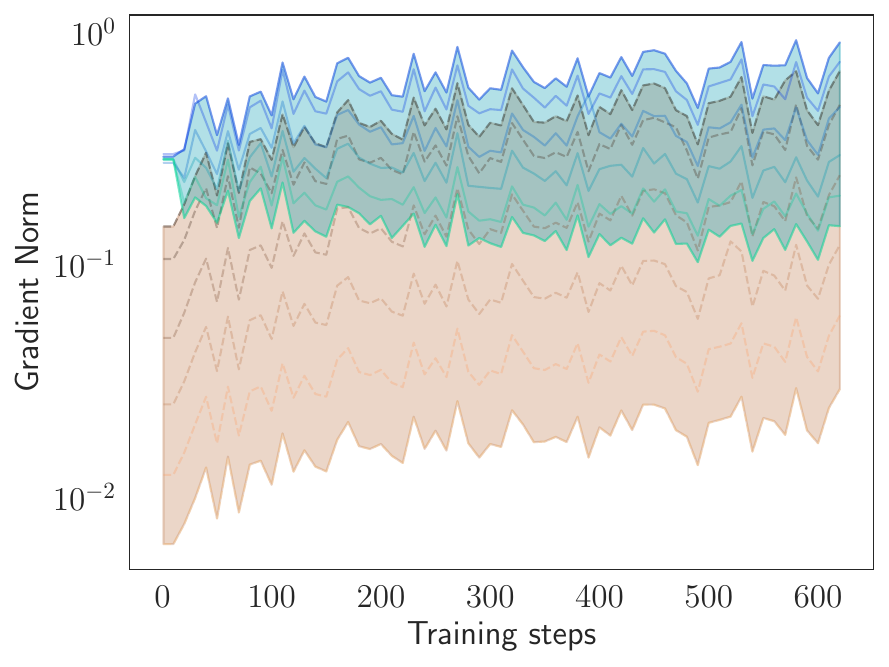}
    \end{minipage}
    \caption{\textbf{Llama 2 7B with adapters only in attention modules}. We vary ranks $r\in\{4,8,32,128,512,2048\}$. Blue-green gradient is used for rsLoRA, and copper for LoRA, where for each, brighter color indicates higher rank. In the gradient norm plot, we also shade in the region in between the minimum and maximum rank for each method in its corresponding color scheme. We observe very similar results as the rest of our experiments, where we train with adapters in all MLP and attention modules.
    }
    \label{fig:attnonly}
\end{figure}

\subsection{Other Scaling Factors than in LoRA and rsLoRA}
To confirm instability when using a scaling factor larger than $\frac{1}{\sqrt r}$, we checked the setting of 
\begin{equation}
    \gamma_r=\frac{1}{r^{1/4}}.
\end{equation}

Also, to see if slowdown of learning turns into outright collapse when we reduce $\gamma_r$ more than in LoRA, we check setting 
\begin{equation}
    \gamma_r=\frac{1}{r^{2}}.
\end{equation}

\begin{figure}[h]
    \centering
    \begin{minipage}{.55\linewidth}
        \includegraphics[width=1.0\linewidth]{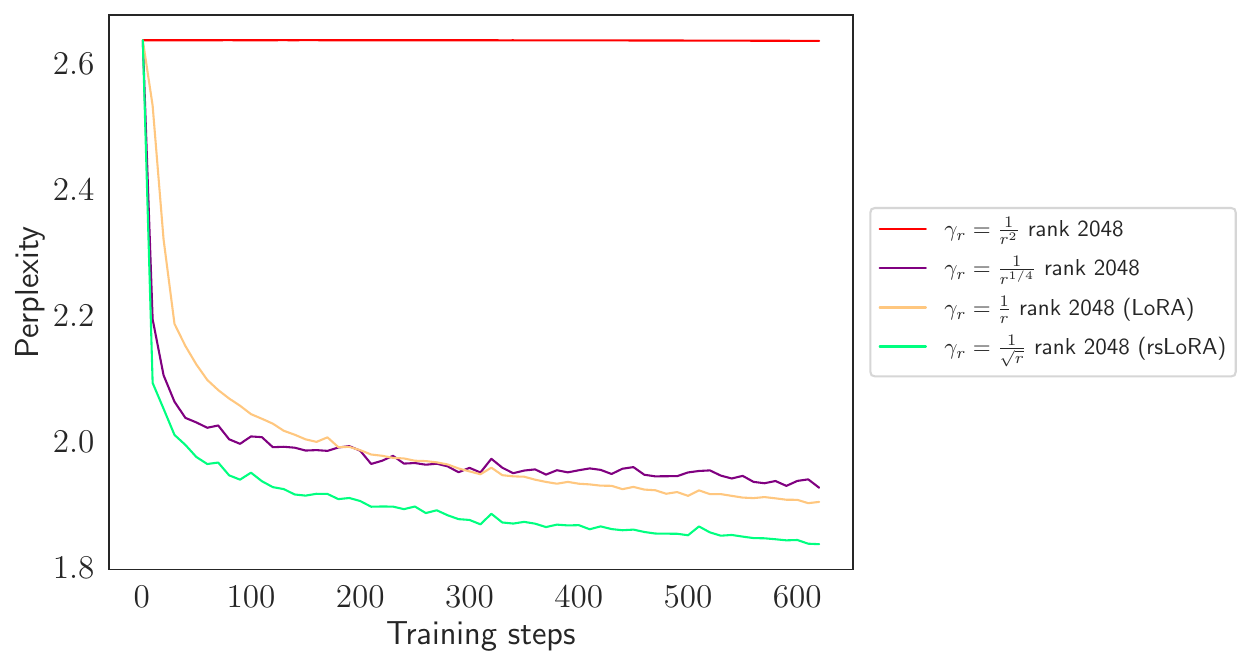}
    \end{minipage}
    \caption{\textbf{Other scaling factors compared to rsLoRA and LoRA} for rank 2048.}
    \label{fig:otherscaling}
\end{figure}

\newpage

\subsection{Activations}

Since theorem \ref{mainthrm} made statements about the moments of the activations, we inspected the averaged first two moments of post-adapter pre-LayerNorm activations. We found that both LoRA and rsLoRA seemed to have relatively stable and non-collapsing activations when averaged over all adapters. This could be due a number of things, including: aspects of the transformer's non-adapter architecture like LayerNorm, the fact that activations in either case have very small moments, or averaging out effects from averaging all activations and layers. We do note however that the higher ranks (512, 2048) for rsLoRA can be observed to level off later in training for both moments, and although we do not examine this aspect further, it may be that this is a consequence of rank-stabilization that is helpful for training with higher ranks.

\begin{figure}[!h]
    \centering
    \hspace{2.1cm}
    \begin{minipage}{.45\linewidth}
        \includegraphics[width=1.0\linewidth]{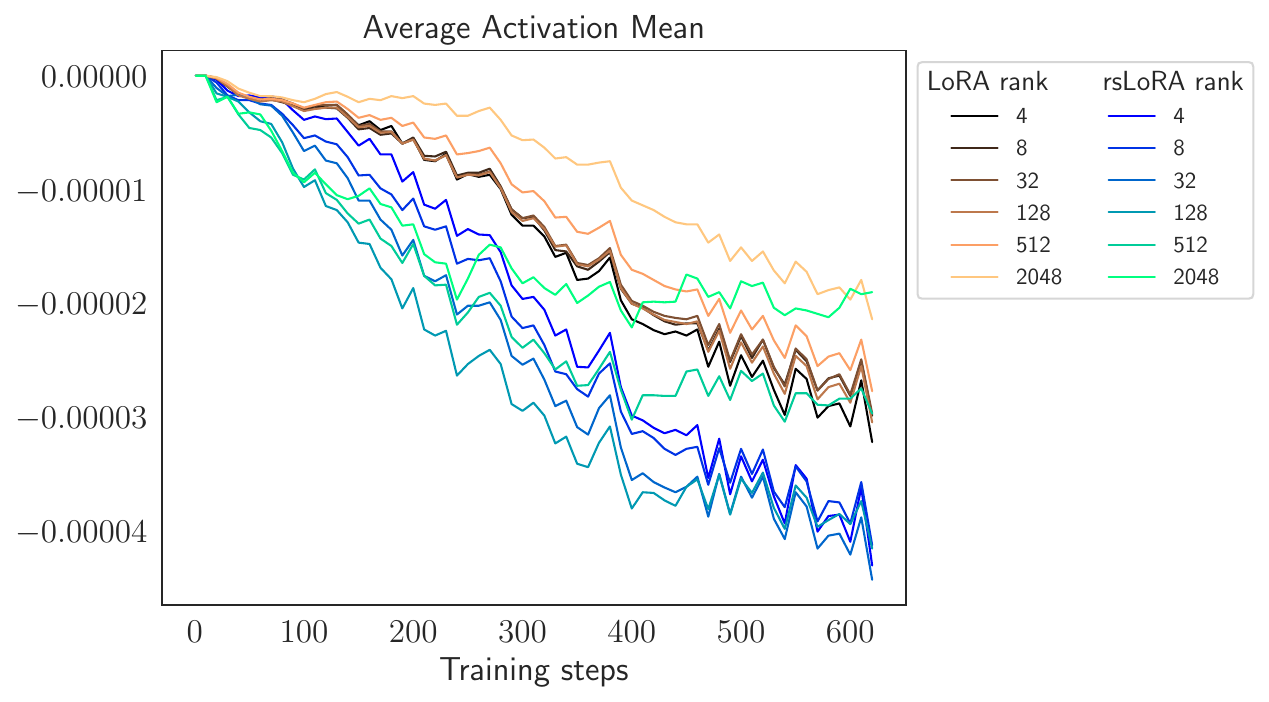}
    \end{minipage}
    \newline
    \centering
    \begin{minipage}{.45\linewidth}
        \includegraphics[width=1.0\linewidth]{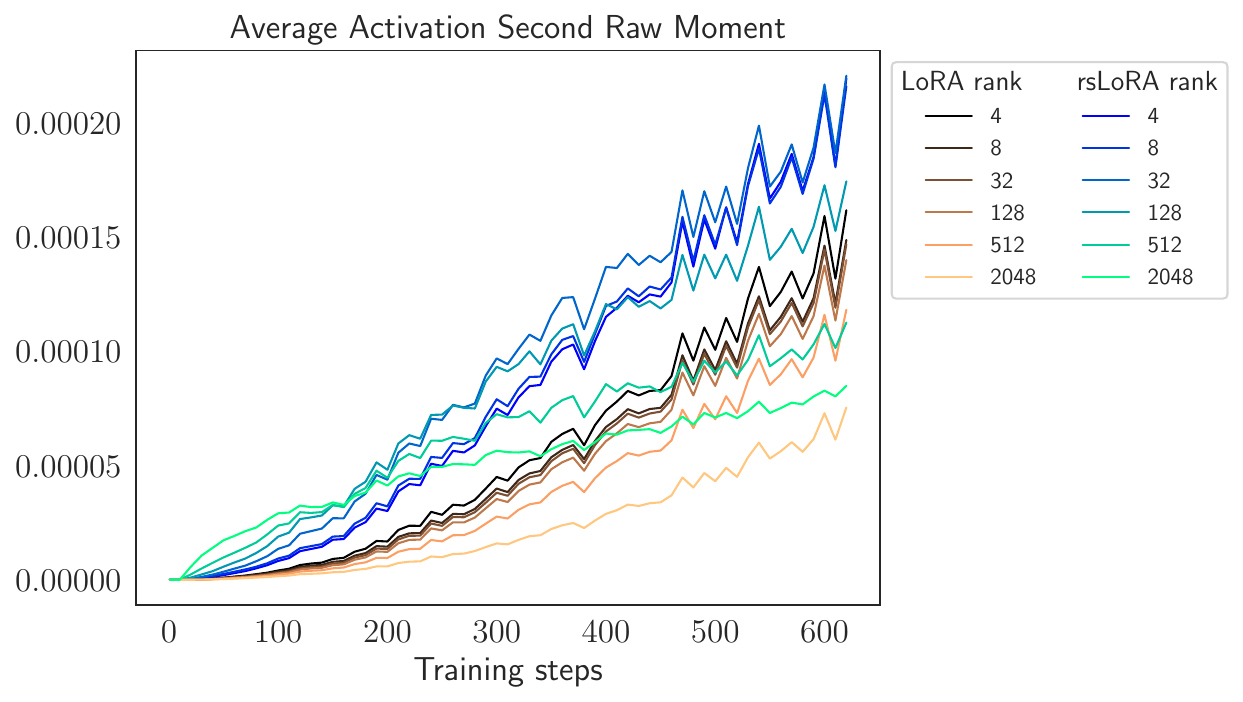}
    \end{minipage}
    \caption{\textbf{Average activation first and second moments over training.} Here we plot the average activation moments with LoRA and rsLoRA. We see that the average moments for both methods look benign and small in magnitute. One does observe a difference with high ranks in rsLoRA, where the for both moments, the increase in magnitute over training steps levels off considerably.
    }
    \label{fig:actmoms}
\end{figure}


\end{document}